\documentclass[letterpaper, 10 pt, conference]{ieeeconf}  

\IEEEoverridecommandlockouts                              

\usepackage{graphics} 
\usepackage{epsfig} 
\usepackage{times} 
\usepackage{amsmath} 
\usepackage{amssymb}  
\usepackage{caption}
\usepackage{subcaption}
\usepackage{tikz,pgfplots}
\usepackage{aircraftshapes}
\usepackage{comment}
\usepackage{tikz,amsmath}
\usepackage{tikz-3dplot}
\usepackage{hyperref}
\definecolor{darkblue}{rgb}{0.0,0.0,0.3}
\hypersetup{colorlinks,breaklinks,
            linkcolor=black,urlcolor=darkblue,
            anchorcolor=black,citecolor=black}
\usetikzlibrary{3d,shapes,arrows,calc,patterns,positioning,decorations.pathmorphing,decorations.markings}

\newcommand{\dfb}{\stackrel{\Delta}{=}}
\newcommand{\R}{\ensuremath{\mathbb R}}

\newtheorem{theorem}{\textbf{Theorem}}[section]
\newtheorem{proposition}{\textbf{Proposition}}[section]
\newtheorem{remark}[theorem]{Remark}

\tikzstyle{block} = [draw, fill=white, rectangle, 
	minimum height=3em, minimum width=6em, text width=5em, text centered]
\tikzstyle{sum} = [draw, fill=white, circle, node distance=2cm]
\tikzstyle{input} = [coordinate]
\tikzstyle{output} = [coordinate]
\tikzstyle{pinstyle} = [pin edge={to-,thin,black}]
\tikzstyle{branch} = [circle,inner sep=0pt,minimum size=1mm,fill=black,draw=black]
\tikzstyle{bus} = [draw, fill=black, rectangle, minimum height=3em, minimum width=0.5em]

\tikzstyle{vertex}=[circle,fill=black!25,minimum size=20pt,inner sep=0pt]
\tikzstyle{selected vertex} = [vertex, fill=red!24]
\tikzstyle{edge} = [draw,thick,-]
\tikzstyle{dedge} = [draw,thick,->]
\tikzstyle{shadowdedge} = [draw, dotted,->]
\tikzstyle{weight} = [font=\small]
\tikzstyle{selected edge} = [draw,line width=5pt,-,red!50]
\tikzstyle{ignored edge} = [draw,line width=5pt,-,black!20]

\title{\LARGE \bf
Flexible collaborative transportation by a team of rotorcraft
}

\author{Hector Garcia de Marina$^{1}$ Ewoud Smeur$^{2}$
\thanks{$^{1}$H. Garcia de Marina is with the Unmanned Aerial Systems Center, Southern University of Denmark, Denmark. {\tt\small hgm@mmmi.sdu.dk}.}%
\thanks{$^{2}$E. Smeur is with the MAVLab, Department of Aerospace Engineering, Delft University of Technology, 2629HS Delft, The Netherlands.}%
}
\begin{document}

\maketitle
\thispagestyle{empty}
\pagestyle{empty}

\begin{abstract}
We propose a combined method for the collaborative transportation of a suspended payload by a team of rotorcraft. A recent distance-based formation-motion control algorithm based on assigning distance disagreements among robots generates the acceleration signals to be tracked by the vehicles. In particular, the proposed method does not need global positions nor tracking prescribed trajectories for the motion of the members of the team. The acceleration signals are followed accurately by an Incremental Nonlinear Dynamic Inversion controller designed for rotorcraft that measures and resists the tensions from the payload. Our approach allows us to analyze the involved accelerations and forces in the system so that we can calculate the worst case conditions explicitly to guarantee a nominal performance, provided that the payload starts at rest in the 2D centroid of the formation, and it is not under significant disturbances. For example, we can calculate the maximum safe deformation of the team with respect to its desired shape. We demonstrate our method with a team of four rotorcraft carrying a suspended object two times heavier than the maximum payload for an individual. 
Last but not least, our proposed algorithm is available for the community in the open-source autopilot Paparazzi.
\end{abstract}

\section{Introduction}
Robot swarms are envisioned to assist humans in logistics operations with better cost-effective approaches \cite{yang2018grand}. Rotorcraft are an essential part of this vision due to their mechanical simplicity and relative expendable nature. In this paper, we demonstrate a systematic method to deal with the transport of objects that are too heavy for a single rotorcraft, but not for a team of them. This added complexity is supported by the fact that rotorcraft do not scale up well aerodynamically speaking, and they become dangerous to operate if their size is big enough. In particular, we focus on carrying suspended loads. In this way, we give flexibility to the team to change its shape and distribute the load among the vehicles efficiently.

Formation control has been employed in the cooperative transportation of objects on the ground \cite{wang2004control}. In our approach, we use rigidity theory \cite{AnYuFiHe08} to describe the geometrical shape of the team so that we can design the desired load distribution among the rotorcraft in their steady-state, even while moving. The coordinated motion of the team is induced by injecting disagreements into the inter-robot distances of the formation. The superposition of specific sets of disagreements creates translational, rotational, and scaling motions of a distributed rigid formation \cite{MaJaCa15}. We combine the team motion controller with Incremental Nonlinear Dynamic Inversion (INDI) to track desired accelerations in the individual vehicles \cite{smeur2015adaptive}. By having prior knowledge about the created forces and moments of the actuators and their delays, the INDI tracks the commanded accelerations coming from the formation controller by using acceleration measurements, which includes the tension force from the load. This measurement allows the INDI controller to calculate the necessary increment in the previous control action for the motors. 

Our proposed method brings substantial positive differences with respect to the current approaches in the literature on collaborative transportation with aerial vehicles. For example, we do not need centralized calculations or global positioning, nor do we require previously calculated paths to be tracked by the vehicles \cite{jiang2013inverse,sreenath2013dynamics,michael2011cooperative}. Therefore, the usage of accurate motion capture systems is not a crucial requirement for our approach. We will see that while we employ them in our demonstration, we do not require global positioning but we fake relative measurements that could be provided by an onboard system \cite{coppola2018board}. Another practical positive aspect of our approach is that we promote scalability. The algorithm can be executed in a distributed way, i.e., robots can calculate their own desired accelerations based only on local information, and the calculations can be done in simple microcontrollers. Therefore we do not require heavy computers for the practical task like in \cite{tagliabue2017collaborative}. We do not require the individual vehicles to track any trajectories either. For example, such an approach, by a centralized motion planning, seems pivotal for the motion of the rotorcraft team in many works whose implementation is shown only in simulations \cite{masone2016cooperative,tagliabue2017collaborative,lee2018geometric}. Nevertheless, they consider active control over the payload. 

The main contribution of this paper focuses on the practical demonstration of the combined approaches of the INDI controller together with the motion of the formation by disagreements for the collaborative transportation of objects. In particular, we will show how these two techniques can predict results on the deformation of the shape together with the control of the velocity of the formation.

This paper is divided as follows. In Section \ref{sec: guidance} we introduce the guidance system for the formation and motion of the rotorcraft. The controller responsible for tracking the guidance signals is described in \ref{sec: indi}. We continue by presenting in Section \ref{sec: wcc} a qualitative description of the forces involved in the system and how to quantitatively calculate tolerances in the formation and control gains from a worst case condition. We show the performance of our system in Section \ref{sec: exp} with a team of four rotorcraft transporting a heavy object for an individual vehicle. We finally end the paper in Section \ref{sec: con} with some conclusions.

\section{Guidance system for the formation-motion control of second-order robots}
\label{sec: guidance}


\subsection{Undirected rigid formations}
Consider a team of $n\geq 2$ robots and denote by $p_i\in\mathbb{R}^2, i\in\{1,\dots,n\}$ their 2D positions in the horizontal plane parallel to the ground with respect to a fixed navigation frame of coordinates. The guidance system for the formation-motion of the team generates an acceleration signal to be tracked by the robots. In particular, the guidance will assume that the vehicles can generate such acceleration \emph{sufficiently fast} in the fixed navigation frame as it is a common assumption in the literature \cite{lee2018geometric}. Although our rotorcraft (quadcopter) can only generate a force along one axis in its body frame, we will see that the INDI controller adapts the attitude and the thrust of the vehicle fast enough so that the vehicle can track a \emph{sufficiently slow varying} acceleration signal effectively, while it rejects disturbances such as the non-constant tension from the ropes. In particular, as we will see, the exponential nature of the proposed guidance and the INDI control will allow us to set the time constants for the exponential decay of their signals by appropriately selecting their gains. The guidance system considers the following dynamical model for the formation-motion task of the vehicles
\begin{equation}
	\begin{cases}
\dot p = v \\
\dot v = u
	\end{cases},
\label{eq: dyn}
\end{equation}
where $p,v\in\mathbb{R}^{2n}$ are the stacked vector of positions and velocities in the plane parallel to the ground respectively, and $u\in\mathbb{R}^{2n}$ is the stacked vector of accelerations generated by the guidance system.

A robot does not need to measure its relative position with respect to all the robots in the team, but only with respect to its \emph{neighbors}. The neighbors' relationships are described by an undirected graph $\mathbb{G} = (\mathcal{V}, \mathcal{E})$ with the vertex set $\mathcal{V} = \{1, \dots, n\}$ and the ordered edge set $\mathcal{E}\subseteq\mathcal{V}\times\mathcal{V}$. The set $\mathcal{N}_i$ of the neighbors of robot $i$ is defined by $\mathcal{N}_i\dfb\{j\in\mathcal{V}:(i,j)\in\mathcal{E}\}$. We define the elements of the incidence matrix $B\in\R^{|\mathcal{V}|\times|\mathcal{E}|}$ for $\mathbb{G}$ by
$b_{ik} = \{1,-1\}$ if $\{\mathcal{E}_k^{\text{tail}},\mathcal{E}_k^{\text{head}}\}$ or $0$ otherwise, where $\mathcal{E}_k^{\text{tail}}$ and $\mathcal{E}_k^{\text{head}}$ denote the tail and head nodes, respectively, of the edge $\mathcal{E}_k$, i.e., $\mathcal{E}_k = (\mathcal{E}_k^{\text{tail}},\mathcal{E}_k^{\text{head}})$. For undirected graphs, how one sets the direction of the edges is not relevant for the stability results or the practical implementation of the algorithm \cite{oh2015survey}. The proposed formation control algorithm is based on the distance-based approach, i.e., we are defining shapes by only controlling distances between neighboring robots. These shapes are based on the rigidity graph theory \cite{AnYuFiHe08}. One of the advantages of controlling distances instead of relative positions is the freedom of the shape to be rotated and translated without modifying the controller. This fact will give us enough freedom to design both the motion and formation controller at once.

The stacked vector of the sensed relative positions by the robots can be calculated as
\begin{equation}
	z = (B^T \otimes I_2)p,
\end{equation}
where $I_2$ is the $2\times 2$ identity matrix, and the operator $\otimes$ denotes the Kronecker product. Note that each vector $z_k = p_i - p_j$ stacked in $z$ corresponds to the relative position associated with the edge $\mathcal{E}_k = (i, j)$. The introduced concepts and notations are illustrated in Figure \ref{fig: rigid}. As an example, in this paper we will focus on a regular square with the neighbors defined by the incidence matrix
\begin{equation}
	B = \left[\begin{smallmatrix}
	1  &  0  &  0 & 1  &  1 & 0\\
	-1 &  1  &  0 & 0  &  0 & 1\\
	0  &  0  &  1 & 0  & -1 & -1\\
	0  & -1  & -1 & -1 &  0 & 0
	\end{smallmatrix}\right].
\label{eq: B}
\end{equation}
Let $d := \begin{bmatrix}d_1,\dots,d_k\end{bmatrix}^T, k\in\{1,\dots,|\mathcal{E}|\}$ be the stacked column vector of fixed distances, associated to their corresponding edges, which defines the desired regular square. Then, the error signals to be minimized are given by
\begin{equation}
	e_k(t) := ||z_k(t)|| - d_k.
	\label{eq: e}
\end{equation}
The control action for each robot in order to stabilize the regular square can be derived from the gradient descent of the potential function involving all the error distances to be minimized and the kinetic energy of the agents
\begin{equation}
	V =\frac{c_1}{2}\sum_{i=1}^{|\mathcal{V}|}||v_i||^2 + \frac{c_2}{2} \sum_{k=1}^{|\mathcal{E}|} (||z_k(t)|| - d_k)^2, \label{eq: Vkquad}
\end{equation}
where $c_1,c_2 >0$, which leads to the following control action \cite{de2017taming} for each robot $i$
\begin{equation}
	^iu_i = -c_1{^iv}_i -c_2\sum_{j\in\mathcal{N}_i}\frac{^i(p_i-p_j)}{||p_i-p_j||}(||p_i-p_j|| - d_{(i,j)}), \label{eq: ui}
\end{equation}
where each desired distance $d_{(i,j)} = d_{(j,i)}$ is associated with its corresponding $d_k$, and the superscript $i$ over the vectorial quantities is used for the representation of a vector with respect to the local frame of coordinates of robot $i$.
\subsection{Motion control}
The control action (\ref{eq: ui}) only achieves the task of converging (exponentially fast) to the desired shape where all the agents will be stopped \cite{de2017taming}. In order to create motion we ask the agents to disagree on the distances $d_{(i,j)}$, i.e., the second term in (\ref{eq: ui}) is written as
\begin{align} \scriptstyle
	c_2& \scriptstyle \sum_{j\in\mathcal{N}_i}\frac{^i(p_i-p_j)}{||p_i-p_j||}\left (||p_i-p_j|| - (d_{(i,j)} + \mu_{(i,j)})\right) = \nonumber \\
	& \scriptstyle{=c_2\sum_{j\in\mathcal{N}_i}\frac{^i(p_i-p_j)}{||p_i-p_j||}(||p_i-p_j|| - d_{(i,j)}) - c_2\sum_{j\in\mathcal{N}_i}\frac{^i(p_i-p_j)}{||p_i-p_j||} \mu_{(i,j)},
	\label{eq: dis}}
\end{align}
where $\mu_{ij}\in\mathbb{R}$ is a designed disagreement and it is equal to zero in (\ref{eq: ui}). Note that the second term in (\ref{eq: dis}) will be the responsible of the desired motion of the agent $i$ during its steady state \cite{de2017taming}, e.g., once the regular square is achieved. In particular, these acceleration signals are linear combinations of the unit relative positions $\frac{^iz_k}{||z_k||}$ where the coefficients are the disagreements $\mu_{(i,j)}$.

We design three rigid motions for the team: two orthogonal translations and one rotational following the systematic method explained in \cite{de2017taming}. For each of these motions, each rotorcraft must follow a specific velocity with respect to the body frame $O_b$ as illustrated in Figure \ref{fig: rigid}. We first write the desired velocities for each motion as linear combinations of $\frac{^iz_k}{||z_k||}$ so we can construct the following matrix with disagreements based on the incidence matrix (\ref{eq: B})
\begin{equation}
	A_v = \left[\begin{smallmatrix}
	\mu_{(1,2)}  &  0  &  0 & \mu_{(1,4)}  &  \mu_{(1,3)} & 0\\
	\mu_{(2,1)} &  \mu_{(2,4)}  &  0 & 0  &  0 & \mu_{(2,3)}\\
	0  &  0  &  \mu_{(3,4)} & 0  & \mu_{(3,1)} & \mu_{(3,2)}\\
	0  & \mu_{(4,2)}  & \mu_{(4,3)} & \mu_{(4,1)} &  0 & 0
	\end{smallmatrix}\right].
\label{eq: Av}
\end{equation}

\begin{figure}
\centering
	\begin{tikzpicture}[scale=0.75]
		\node [quadcopter top,minimum width=1cm] at (-1.5,-1.5) {};
		\node at (-2.2,-2.2) {$3$};
		\node [quadcopter top,minimum width=1cm,rotate=20] at (-1.5,1.5) {};
		\node at (-2.2,2.2) {$4$};
		\node [quadcopter top,minimum width=1cm,rotate=-20] at (1.5,-1.5) {};
		\node at (2.2,-2.2) {$2$};
		\node [quadcopter top,minimum width=1cm,rotate=0] at (1.5,1.5) {};
		\node at (2.2,2.2) {$1$};
		\fill[gray] (-0.2,-0.2) rectangle (0.2,0.2) {};
		\draw [dashed] (1.5,1.5) -- (1.5,-1.5) node[pos=0.5,right] {$||z_1||$};
		\draw [dashed] (1.5,-1.5) -- (-1.5,1.5) node[pos=0.25,above,sloped] {$||z_2||$};
		\draw [dashed] (-1.5,1.5) -- (-1.5,-1.5) node[pos=0.5,left] {$||z_3||$};
		\draw [dashed] (-1.5,1.5) -- (1.5,1.5) node[pos=0.5,above,sloped] {$||z_4||$};
		\draw [dashed] (-1.5,-1.5) -- (1.5,1.5) node[pos=0.25,above,sloped] {$||z_5||$};
		\draw [dashed] (-1.5,-1.5) -- (1.5,-1.5) node[pos=0.5,above,sloped] {$||z_6||$};
		\draw [blue, thick, ->] (-1.5,-1.5) -- (-0.5,-1.5);
		\draw [blue, thick, ->] (-1.5,1.5) -- (-0.5,1.5);
		\draw [blue, thick, ->] (1.5,-1.5) -- (2.5,-1.5);
		\draw [blue, thick, ->] (1.5,1.5) -- (2.5,1.5);
		\draw [red, thick, ->] (-1.5,-1.5) -- (-1.5,-0.5);
		\draw [red, thick, ->] (-1.5,1.5) -- (-1.5,2.5);
		\draw [red, thick, ->] (1.5,-1.5) -- (1.5,-0.5);
		\draw [red, thick, ->] (1.5,1.5) -- (1.5,2.5);
		\draw [green, thick, ->] (-1.5,-1.5) -- (-2.5,-0.5);
		\draw [green, thick, ->] (-1.5,1.5) -- (-0.5,2.5);
		\draw [green, thick, ->] (1.5,-1.5) -- (0.5,-2.5);
		\draw [green, thick, ->] (1.5,1.5) -- (2.5,0.5);

		\draw [thick, ->] (0,0) -- (0,1) node[pos=0, below] {$O_b$};
		\draw [thick, ->] (0,0) -- (1,0);
\end{tikzpicture}
	\caption{\small We ask the rotorcraft to control their distances in between to form a regular square, where the payload is placed at the centroid but in a different altitude. The red, blue and green color vectors are velocities that create translational and rotational motions of the formation with respect to $O_b$. They are constructed as linear combinations of $z_k$ with the disagreements $\mu_{ij}$.}
	\label{fig: rigid}
\end{figure}
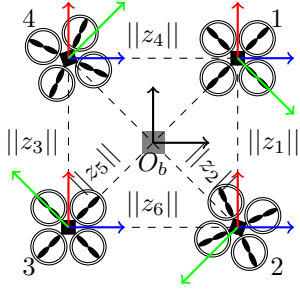

Let us write in compact form the stacked vector of all control actions (\ref{eq: ui}) with disagreements as in (\ref{eq: dis}) in the navigation frame 
\begin{equation} 
	u = -c_1 v - c_2\overline BD_zD_{\tilde z}e + \overline A\overline D_{\tilde z}z,
	\label{eq: udi}
\end{equation}
where $\tilde z\in\mathbb{R}^{|\mathcal{E}|}$ is the stacked vectors of all $\frac{1}{||z_k||}$, we define the operator $\overline X := X \otimes I_2$, and $D_x$ is the operator that takes each stacked element (vector or scalar) of $x$ and places them in a diagonal block matrix. The matrix $A$ will be function of the previously calculated $A_v$. In particular, we have dropped $c_2$ in the third element of (\ref{eq: udi}) since, as we will see, it can be cancelled out by scaling up or down the disagreements in (\ref{eq: Av}).

We need to introduce different calculations with respect to \cite{de2017taming} in order to figure out the matrix $A$ since in this work we deal with a different potential function (\ref{eq: Vkquad}) that leads to work with the unit vectors $\frac{z_k}{||z_k||}$ in the control action (\ref{eq: udi}). We first need to define the velocity error $e_v := v - \overline A_v\overline D_{\tilde z} z$.
\begin{proposition}
	\label{pro: 1}
	The matrix $A$ is calculated from $A_v$ and is given by
	$A = c_1 A_v + A_a$, where $A_a = A_{v_r}D_dB^TA_{v_r}$, and $A_{v_r}$ defines the desired steady-state rotational motion.
\end{proposition}
\begin{proof}
Let us define the stacked vector $v^*(t)\in\mathbb{R}^{2|\mathcal{V}|}$ of desired velocities for the agents. In particular, we have designed it from linear combinations of the unit vectors of the desired relative positions by employing (\ref{eq: Av}). Therefore we can write
\begin{equation}
	v^*(t) = \overline A_v\overline D_d z^*(t),
\end{equation}
	where we recall that $d$ is the stacked vector of desired distances and $z^*\in\mathbb{R}^{2|\mathcal{E}|}$ is the stacked vector of desired relative positions with respect to $O_b$ in Figure \ref{fig: rigid}. Now we calculate the stacked vector with all the desired accelerations
\begin{align}
	\frac{\mathrm{d}}{\mathrm{d}t}v^*(t) &= \overline A_v \overline D_d \frac{\mathrm{d}}{\mathrm{d}t} z^*(t) \nonumber \\
	&= \overline A_v \overline D_d   \overline B^T v^*(t) \nonumber \\
	&= \overline A_v \overline D_d  \overline B^T \overline A_v\overline D_d z^*(t). \label{eq: p1}
\end{align}
Let $A = c_1A_v + A_a$ so that (\ref{eq: udi}) can be written as
	\begin{equation}
	u = -c_1e_v -c_2\overline BD_zD_{\tilde z}e + \overline A_a\overline D_{\tilde z}z, 
	\label{eq: p2}
\end{equation}
	therefore once $e,e_v=0$, i.e., the formation is at the desired shape and velocity, we have that $u^*(t) = \overline A_a D_dz^*(t)$, and noting that for pure translational motion the accelerations are identically zero because $\frac{\mathrm{d}}{\mathrm{dt}}z^*(t) = 0$ in such a case, then from (\ref{eq: p1}) and (\ref{eq: p2}) we have that
\begin{align}
	A_a D_d z^*(t) &= A_{v_r}D_dB^TA_{v_r} D_d z^*(t) \nonumber \\
	A_a &= A_{v_r}D_dB^TA_{v_r}.
\end{align}
\end{proof}
The key for showing the exponential stability of $e$ and $e_v$ lies on the fact that under the control (\ref{eq: udi}) for $A = 0$ these signals are exponentiall stable \cite{de2017taming}. Then, the disagreements, i.e. $A\neq0$, are treated as parametric disturbances as suggested in \cite{garcia2015controlling}. Therefore, if they are small enough (or $c_1$ and $c_2$ big enough), then we do not modify the exponential nature of the convergence of $e$ and $e_v$.

\section{Rotorcraft control}
\label{sec: indi}

\subsection{Incremental Nonlinear Dynamic Inversion controller}
We have taken the concept of Incremental Nonlinear Dynamic Inversion (INDI) \cite{smeur2018cascaded}, and applied it to track the linear accelerations generated by the guidance system for the formation-motion of the rotorcraft. The idea behind INDI is that forces and moments acting on the vehicle are, according to classical mechanics, proportional to the acceleration and angular acceleration of the vehicle. The acceleration can be measured with the accelerometer, and the angular acceleration can be derived from the gyroscope. From a desired increment in linear and angular acceleration, the required increment in inputs can be easily calculated using a control effectiveness matrix. Since the 3D tensions from the rope will be measured by the accelerometer and gyroscope as well, they will be naturally counteracted by the controller, without having to model the load. Though the tension of the ropes may introduce small moments on the rotorcraft, the ropes are attached close to the center of mass and the expected effect on the attitude dynamics is negligible.

To obtain a measurement of the linear acceleration $a_0$ of a rotorcraft, we take the specific force measurement of the accelerometer, and add the gravity vector to that. The acceleration a small time step ahead can now be predicted with
\begin{equation}
	a - a_0 = \frac{1}{m}G(\eta_0,T_0)(u-u_0),
	\label{eq: ainc}
\end{equation}
where $a\in\mathbb{R}^3$ is the acceleration in the navigation frame (the subscript $0$ denotes for current value), $m$ is the mass of the Bebop, $\eta = \begin{bmatrix}\phi & \theta & \psi \end{bmatrix}$ are the three attitude angles \emph{roll}, \emph{pitch} and $\emph{yaw}$, we group the pitch, roll and thrust in $u = \begin{bmatrix}\phi & \theta & T\end{bmatrix}^T$, and we finally define the matrix of partial derivatives of the thrust vector 
\begin{equation}
	G(\eta,T) =	\left[\begin{smallmatrix}
		T(c\phi s\psi - s\phi c\psi s\theta) & Tc\phi c\psi c\theta & s\phi s\psi + c\phi c\psi s\theta \\
		-T(s\phi s\psi s\theta + c\phi c\psi) & Tc\phi s\psi c\theta & c\phi s\psi s\theta - s\phi c\psi \\
		-Tc\theta s\phi & -Ts\theta c\phi & c\phi c\theta
	\end{smallmatrix}\right].
\end{equation}
The measured acceleration $a_0$ incorporates disturbances, and the force from the rope tension. Because the acceleration measurement is noisy due to vibrations, we employ a low pass Butterworth filter, denoted with a subscript $f$. In order to synchronize the signals, $a_0$ and $u_0$ will both be filtered such that they will have the same delay. The filtered signals are incorporated in (\ref{eq: ainc}) and the equation is inverted to obtain the following incremented controller
\begin{equation}
	\Delta u = mG^{-1}(\eta_f,T_f)(\nu_a - a_f),
	\label{eq: inca}
\end{equation}
where $\nu_a$ is the desired 3D linear acceleration for the rotorcraft in the navigation frame. In particular, the first two components of $\nu_a$ for each rotorcraft are given by (\ref{eq: udi}) and the third or vertical component is generated by the PD controller
\begin{equation}
	\nu_{a_z} = k_p(p_z - p_{z_d}) - k_vv_z,
	\label{eq: pdalt}
\end{equation}
where the gains $k_p$ and $k_v$ are chosen according to the stability analysis given in \cite{smeur2018cascaded}.
\begin{remark}
	To calculate $G$ in (\ref{eq: inca}) we need to know $T_f$, which is experimentally estimated by a quadratic function $f_T(\omega_f^2)$ in a static airflow regime. Errors due to the simplifications in the modelling are expected to have a low impact on the performance. This is explained because of the incremental nature of the controller, i.e., if an increment of thrust does not give the desired acceleration, then another increment is applied in a similar way as an integral controller does. However, note that the INDI measures the disturbances and the tensions from the rope, and have a \emph{feedforward} knowledge of the vehicle and its actuators.
\end{remark}

\subsection{Effect of the payload on the control}
As we have discussed, no explicit knowledge about the payload is incorporated in the control of the rotorcraft since the
 INDI controller is naturally designed for dealing with non-modelled external forces by measuring them directly. For example, due to the incremental nature of the controller, feedforward control increments will be applied on the actuators as long as the desired acceleration is not reached, which means that steady state offsets will not occur. Consequently, external forces are incorporated in the control either by compensating them, or by taking advantage from them if for example the tension of the load and the desired acceleration have the same direction.

\begin{figure}
\centering
\begin{tikzpicture}[auto, node distance=2cm,>=latex']
	\node [block] (bebop) {Motors \\ Rigid body};
	\node [block, below of=bebop, text width=8em, minimum height=1em, xshift=-0em, yshift=2em] (kf) {KF attitude};
	\node [block, below of=kf, text width=8em, minimum height=1em, yshift=3.3em] (alt) {PD alt. Eq. (\ref{eq: pdalt})};
	\node [block, below of=alt, text width=8em, minimum height=1em, yshift=4em] (gnc) {Guidance Eq. (\ref{eq: udi})};
	\node [block, left of=bebop, xshift=-3em] (indiw) {Attitude Control (INDI)};
	\node [block, left of=alt, minimum height=3em, xshift=-5em, yshift=0.4em] (india) {INDI $a$ Eq. (\ref{eq: inca})};

	\draw[->] ($(bebop.north east)!0.10!(bebop.south east)$) -- +(1.2,0) node[pos=0.75,above]{$p,v$} |- ($(gnc.north east)!0.25!(gnc.south east)$);
	\draw[->] ($(bebop.north east)!0.10!(bebop.south east)$) -- +(1.2,0) |- node[pos=0.25]{}(alt.east);
	\draw[<-] ($(gnc.north east)!0.75!(gnc.south east)$) -- +(1.2,0) node[pos=0.5,below]{$p_{\text{neighbors}}$};
	\draw[->] ($(bebop.north east)!0.90!(bebop.south east)$) -- +(0.75,0) node[pos=0.75, above] {$\Omega$} |- ($(kf.north east)!0.20!(kf.south east)$);
	\draw[->] ($(bebop.north east)!0.90!(bebop.south east)$) -- +(0.75,0) |- +(-5.5,-0.5) |- ($(indiw.south west)!0.20!(indiw.north west)$) node[pos=0.45, left] {$\Omega$};
	\draw[->] (bebop.east) -| (2,0) |- ($(kf.south east)!0.2!(kf.north east)$) node[pos=0.45] {};
	\draw[->] (bebop.east) -| (2,0) node[pos=0.3] {$a$} |- ($(india.south east)!0.75!(india.north east)$);
	\draw[->] (gnc.west) -| +(-0.5,0) |- ($(india.south east)!0.25!(india.north east)$);
	\draw[->] (alt.west) -| +(-0.5,0) |- ($(india.south east)!0.25!(india.north east)$) node[pos=0.7]{$\nu_a$};
	\draw[->] (kf.west) -| +(-0.5,0) |- ($(india.south east)!0.9!(india.north east)$) node[pos=0.5,above]{};
	\draw[->] (kf.west) -- +(-3.5,0) |- (indiw) node[pos=0.75]{$\eta$};
	\draw[->] (indiw.east) -- (bebop.west) node[pos=0.5]{$\omega_c$};
	\draw[->] (india.west) |- +(-0.7,0) |- node[pos=0.7,above] {$\Delta T_d, \eta_d$}($(indiw.south west)!0.85!(indiw.north west)$);
	\draw [color=gray,thick](-1.3, -0.8) rectangle (1.3,1);
	\node at (-1.3,0.75) [above=5mm, right=0mm] {\textsc{Bebop}};
	\end{tikzpicture}
	\caption{\small Block diagram of the elaborated signals in the closed loop system. 
	We actuate over the Bebops commanding desired rpm $\omega_c$ for the motors. We measure the angular velocities ($\Omega$) and linear accelerations ($a$) for the INDI controllers and the attitude estimation. We measure the positions $p$ to fake direct measurements of $z$ with a vision localization system.} 
	\label{fig: block}
\end{figure}
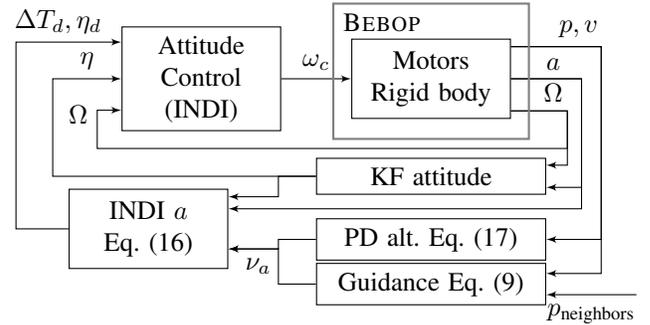

\section{Force and worst case analysis}
\label{sec: wcc}
In this section, as an illustrative example, we conduct an analysis to guarantee the safety of the system while it is close to a steady state, e.g., close to the desired shape with the desired velocity. We analyze the forces involved, in particular the ones in a worst case, and connect them to the physical limits of the vehicles together with the demands from the guidance system in Section \ref{sec: guidance} and the controllers in Section \ref{sec: indi}. For the analysis, our vehicles are four rotorcraft Bebop from the company Parrot that have been characterized in our laboratory \cite{smeur2018cascaded}. Each rotorcraft weighs around $400$ grams, an individual motor/propeller can generate $1.6$ Newtons of thrust once it is rotating at $160$ Hz in hover flight, and a motor/propeller generates a moment around $0.0006$ m/s$^2$/rpm.

We describe the forces involved in the system in a steady state configuration as starting point. We set the desired shape as the regular square in Figure \ref{fig: rigid}, where all the ropes have the same length so that the payload can be placed at the center of the square while hanging. This configuration allows an equal distribution of the payload's weight across all the vehicles. The force diagram on the left of Figure \ref{fig: dia} focuses on the vertical plane connecting vehicles $4$ and $2$. It shows how the vehicles need to tilt to compensate the horizontal components of the ropes' tensions. In particular, we define \emph{tilt} as the angle formed by the thrust force and the horizontal plane parallel to the ground. If the rotorcraft are in equilibrium, $M$ is the mass of the payload, and $l$ is the fixed length of the rope, then the tension satisfies
\begin{equation}
	||T_{M_2}|| = \frac{Mg}{4} \frac{l}{\sqrt{l^2-\frac{d_2^2}{4}}},
	\label{eq: T}
\end{equation}
which increases clearly with $d_2$. Therefore, we will set $d_2$ the shortest as possible depending on the sizes of the rotorcraft and the expected disturbances that will vary $||z_2||$ (equal to $d_2$ in the steady state).

The diagram on the right side of Figure \ref{fig: dia} is in the plane described by the vehicles $1$ and $3$. We describe the situation where we demand the same horizontal acceleration to the vehicles while keeping the altitude constant. In such a case, we see that the force $R$ in vehicle $1$ will create more tension on its rope since the vehicle is more tilted, and consequently it needs a higher $F_1$ to compensate for gravity. As a result, the vehicle $1$ will need to lift more weight from the payload than its neighbor. Because the altitude is constant, and both vehicles equally accelerate, then the payload experiences the same force $R$ as well.

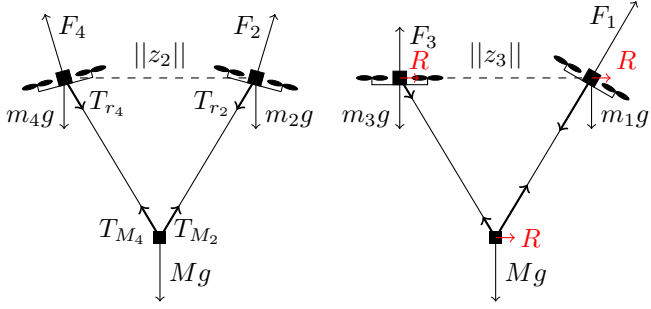
\begin{figure}
	\centering
	\begin{tikzpicture}
		\begin{scope}[scale = 0.85]
		\node [quadcopter side,minimum width=1cm,rotate=15] at (-1.5,0) {};
		\node [quadcopter side,minimum width=1cm,rotate=-15] at (1.5,0) {};
		\fill (-0.1,-0.1-2.5) rectangle (0.1,0.1-2.5);
		\draw (0,-2.5) -- (1.5,0);
		\draw (0,-2.5) -- (-1.5,0);
		\draw [->] (0,-2.5) -- (0,-3.5) node[pos=0.9,above right] {$Mg$};
		\draw [thick, ->] (0,-2.5) -- (0.3,-2) node[pos=0.2, right] {$T_{M_2}$};
		\draw [thick, ->] (0,-2.5) -- (-0.3,-2) node[pos=0.2, left] {$T_{M_4}$};
		\draw [thick, ->] (-1.5,0) -- (-1.2,-0.5) node[pos=0.8, right] {$T_{r_4}$};
		\draw [thick, ->] (1.5,0) -- (1.2,-0.5) node[pos=0.8, left] {$T_{r_2}$};
		\draw [->] (1.5,0) -- (1.5,-0.8) node[pos=0.8, right] {$m_2g$};
		\draw [->] (-1.5,0) -- (-1.5,-0.8) node[pos=0.8, left] {$m_4g$};
		\draw [->] (-1.5,0) -- (-1.8,1) node[pos=0.8, right] {$F_4$};
		\draw [->] (1.5,0) -- (1.8,1) node[pos=0.8, left] {$F_2$};
		\draw [dashed] (-1.5,0) -- (1.5,0) node[pos=0.5, above] {$||z_2||$};
		\end{scope}

		\begin{scope}[scale = 0.85, shift={(5.25,0)}]
		\node [quadcopter side,minimum width=1cm,rotate=0] at (-1.5,0) {};
		\node [quadcopter side,minimum width=1cm,rotate=-30] at (1.5,0) {};
		\fill (-0.1,-0.1-2.5) rectangle (0.1,0.1-2.5);
		\draw (0,-2.5) -- (1.5,0);
		\draw (0,-2.5) -- (-1.5,0);
		\draw [->] (0,-2.5) -- (0,-3.5) node[pos=0.9,above right] {$Mg$};
		\draw [thick, ->] (0,-2.5) -- (0.5,-1.67) node[pos=0.2, right] {};
		\draw [thick, ->] (0,-2.5) -- (-0.2,-2.17) node[pos=0.2, left] {};
		\draw [thick, ->] (-1.5,0) -- (-1.3,-0.33) node[pos=0.8, right] {};
		\draw [thick, ->] (1.5,0) -- (1,-0.83) node[pos=0.8, left] {};
		\draw [->] (1.5,0) -- (1.5,-0.8) node[pos=0.8, right] {$m_1g$};
		\draw [->] (-1.5,0) -- (-1.5,-0.8) node[pos=0.8, left] {$m_3g$};
		\draw [->] (-1.5,0) -- (-1.5,0.8) node[pos=0.8, right] {$F_3$};
		\draw [->] (1.5,0) -- (2.2,1.2) node[pos=0.8, left] {$F_1$};
		\draw [dashed] (-1.5,0) -- (1.5,0) node[pos=0.5, above] {$||z_3||$};
		\draw [red,->] (0,-2.5) -- (0.3,-2.5) node[pos=0.8, right] {$R$};
		\draw [red,->] (-1.5,0) -- (-1.2,0) node[pos=1, above] {$R$};
		\draw [red,->] (1.5,0) -- (1.8,0) node[pos=0.8, above right] {$R$};
		\end{scope}
\end{tikzpicture}
	\caption{\small Diagram forces. On the left, the plane defined by the payload and the vehicles $2$ and $4$ in a configuration of equilibrium. On the right, the plane defined by the payload and the vehicles $1$ and $3$ where the guidance system demands the same acceleration from both vehicles.}
	\label{fig: dia}
\end{figure}

We take the above described cases to calculate conservative bounds on the maximum deformation of the formation, i.e., on the norm of the error signal $e(t)$, on the gains $c_1$ and $c_2$ and the disagreements $\mu_r$ in (\ref{eq: udi}) for a demanded motion once a rotorcraft is close to the equilibrium. We have measured that one actuator of the Bebops can generate around $1.6$ Newtons of force at the $85\%$ of its capacity, so we consider $6.4$ Newtons as a reference thrust force for the following worst case condition. We first estimate the maximum safe tilt for a rotorcraft. During the experiments, the initial positions of the rotorcraft are close to the desired one. Therefore, the payload will always stay close to the centroid as we will see according to the calculations in this section. We consider that a vehicle will share a third of the weight of the payload as a conservative worst case. In particular, we will deal with a payload equal to the mass of a single vehicle. Therefore the maximum vertical force to compensate will be $(0.4 + \frac{0.4}{3})g = 5.22$ Newtons. A simple trigonometric calculation reveals that the maximum tilt for a rotorcraft is $\operatorname{arccos}\frac{5.22}{6.4} = 0.62$ radians. Therefore, we impose a maximum angle of $20$ degrees ($0.35$ rads) for both pitch and roll so that the tilt is below 0.62 rads. Indeed, this relatively small angle helps the INDI controller since for its design we have assumed small increments in the commanded attitude signals. This estimation will leave around $6.4\sin(0.35) = 2.18$ Newtons of maximum available force for a rotorcraft in the horizontal plane.

We then estimate what would be a conservative bound for a horizontal force in opposition to the one created by the thrust of the vehicle. Although present, we will not consider any aerodynamic drag since we will set the maximum vehicle's speed in $1$ m/s, and the tension from the rope is substantially bigger than any drag at that speed. Focusing on vehicle $2$, from equation (\ref{eq: T}) and basic trigonometry we can derive the expression for the bound on the horizontal tension of the rope
\begin{equation} 
	 T_h = \frac{Mg}{3} \frac{||z_2||}{2\sqrt{l^2-\frac{||z_2||^2}{4}}},
	\label{eq: Th}
\end{equation}
where we assumed that the horizontal distance between the vehicle and the payload is approximately $\frac{||z_2||}{2}$ since the initial positions for the rotorcraft are close to the desired square, and the controllers will keep such a situation if they can cope with the predicted worst cases in this current analysis. 
We set the length of the rope to approximately $l=\sqrt{1.25}$ meters, and for equation (\ref{eq: Th}) we consider a very conservative $||z_2|| = 2$ meters since in the experiments we will set $d_2$ close to $1$ meter. As a result, an upper bound to $T_h$ is $\frac{Mg}{3}$, and because the maximum available horizontal force in the worse case for the tilt is $2.18$ Newtons, we can conclude that the maximum norm of the acceleration to be asked by the guidance system is
\begin{equation}
	\operatorname{max}(||\ddot p_i||) = \frac{2.18 - \frac{Mg}{3}}{m_i} = 2.18 \,m/s^2,
\end{equation}
where $m_i = M = 0.4$ Kg are the mass of the vehicle $i$ and the payload. Consequently, we impose on the vehicles that the maximum for each of the acceleration coordinates in the plane horizontal to ground is $2.18\times\sin(0.79) = 1.54 \, m/s^2$, i.e., a vector of magnitude 2.18 with 45 degrees with respect to the X and Y axis. For example, starting from the equation $(\ref{eq: udi})$ for the vehicle $4$ (and the same argument can be extended to the rest of vehicles) we can calculate conservative values for $c_1,c_2$ and $\mu_r$ from the following expression
\begin{equation}
	\begin{cases}
		\ddot p_{4_x} = -c_1 e_{v_x} - c_2(e_4 + 0.7e_2) + \mu_r ||z_4|| \\
		\ddot p_{4_y} = -c_1 e_{v_y} - c_2(e_3 + 0.7e_2) - \mu_r ||z_3||.
	\end{cases}
	\label{eq: wc}
\end{equation}
The first term in the equations in (\ref{eq: wc}) is related to the desired speed in one of the horizontal directions. For example, we recall that for a desired translational motion, the desired velocity is designed by the disagreements $\mu_t$ in Section \ref{sec: guidance}. By design, we will not demand a higher speed than $1$ m/s, so we can safely assume $\operatorname{max}\{||e_v||\} = 1$. The second term refers to the control of inter-vehicle distances. Recall that we start close to the equilibrium with a desired square of side $1$ meter, therefore we can assume a very conservative worst case of $\operatorname{max}\{||e_k||\} = 1$ meter since we consider as worst case for the horizontal tension a distance of $2$ meters between vehicles on the side of the squared formation, e.g., $\operatorname{max}\{||z_4||\} = 2$ meters. Note that looking at Figure \ref{fig: rigid}, $e_2$ almost equally contributes to both components, while for example $e_3$ can be safely omitted in the $x$ component. Finally, the disagreement $\mu_r$ constructs the desired centripetal acceleration towards the centroid, so the formation spins around it. Therefore, in order to satisfy the following worst case condition
\begin{equation}
\scriptstyle 1.54\text{m/s}^2 \geq c_1 \operatorname{max}\{||e_v||\} + c_2(\operatorname{max}\{e_4\}+0.7\operatorname{max}\{e_2\}) + \mu_r \operatorname{max}\{||z_4||\},
	\label{eq: max}
\end{equation}
we assign $c_1 = 0.17$, $c_2 = 0.55$ and $\operatorname{max}\{\mu_r\} = 0.2$, where $\mu_r$ sets the limit to the maximum angular velocity of the spinning motion of the formation. With these chosen values, the system remains stable and within its physical limits. Firstly, the forces from the payload in this analysis have been found substantially smaller than the disturbances that can be handled by the Bebops \cite{smeur2018cascaded} with the INDI controller. Secondly, the incremental accelerations demanded by the guidance system are bounded and within the physical limits to be tracked effectively. Thirdly, the stability analysis of the guidance system \cite{de2017taming} guarantees the maximum magnitude of the error signals $e_v$ and $e$. In particular, if the accelerations are tracked correctly and the distances between the vehicles have an initial error of less than $0.5$ meters, then following \cite{de2017taming} it can be derived that the maximum for a single error distance cannot be more than a meter if the requested velocity has a speed smaller than a meter per second. All in all, the correct performance of the team carrying a load is guaranteed under the described nominal conditions in this section.

\begin{remark}
For chosen the gains $c_1$ and $c_2$, we checked in simulation that the desired acceleration signals have a time constant of around $1$ second for its exponential decay. The chosen gains for the INDI controller indicates a time constant of around $0.01$ seconds. Therefore, we can guarantee that the guidance system is at least 100 times \emph{slower} in order to ensure the stability of the two slow-fast interconnected systems.
\end{remark}

\begin{remark}
The payload is assumed at rest at the initial conditions, e.g., it is not swinging. Since the payload starts close to the centroid of the formation, and the guidance system guarantees that the centroid is under control, then the forces on the payload are very similar to the ones required to achieve the desired velocities of the formation as a whole. Therefore, as we will see in the experiments, we predict to do not have significant swing motions on the payload. If there are small disturbances on it, the INDI would detect them since changes on the tensions of the ropes are measured by the accelerometers. So, both the INDI and the guidance system will compensate such disturbances. 
\end{remark}

\section{Experimental results}
\label{sec: exp}
The experimental results\footnote{A high-definition video of this section can be found at \href{https://www.youtube.com/watch?v=HUZH46Oxc5c}{https://www.youtube.com/watch?v=HUZH46Oxc5c}.} are conducted in a controlled area, e.g., no wind or obstacles around. The Bebops are equipped with the Paparazzi autopilot\footnote{\href{http://wiki.paparazziuav.org/wiki/}{http://wiki.paparazziuav.org/wiki/}}
and their positions are obtained with an Optitrack camera system. An operator with a gamepad commands the movements in Figure \ref{fig: rigid} where the sticks set the magnitude of the disagreements of the matrix $A$ in Proposition \ref{pro: 1}, i.e., the speed of the motions. The operator with the gamepad also controls the scale of the formation and its altitude at once, e.g., to control when the payload is lifted. The guidance system runs in a ground control station. It processes the relative positions of the robots and the input from the gamepad at the fixed frequency of $4$ Hz. We show in Figure \ref{fig: cap} a caption of the experiments and the trajectories described by the team while it is steered by the operator. We would like to remind that the operator commands the formation as a single \emph{super-vehicle}, i.e., the operator maneuvers the team as a single solid rigid body. In this experiment, the payload flies at around an altitude of $1$ m over the ground. The Figure \ref{fig: sig} shows the actual distances between vehicles and their velocities in the navigation frame of coordinates. In particular, we show the distances of the two diagonals and two of the sides of the square as defined in Figure \ref{fig: rigid}. The operator changes the scale of the formation (black-dashed line) over time at the same time he maneuvers the formation. In fact, in order to show the robustness of the system, during some parts of the experiment the operator changes quickly the direction of the desired velocity, e.g., check seconds 253 and 278. One can check that our worst-case analysis holds since the worst error distance we noticed is less than $30$ cm around second 296 during a spinning motion of the formation. In the video of the experiments, one can also check the predicted non-swing motion of the payload during its transportation.
\begin{figure}
\centering
  \includegraphics[width=0.48\columnwidth]{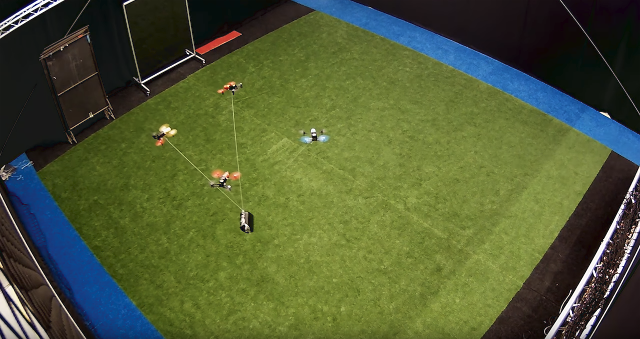}
  \includegraphics[width=0.48\columnwidth]{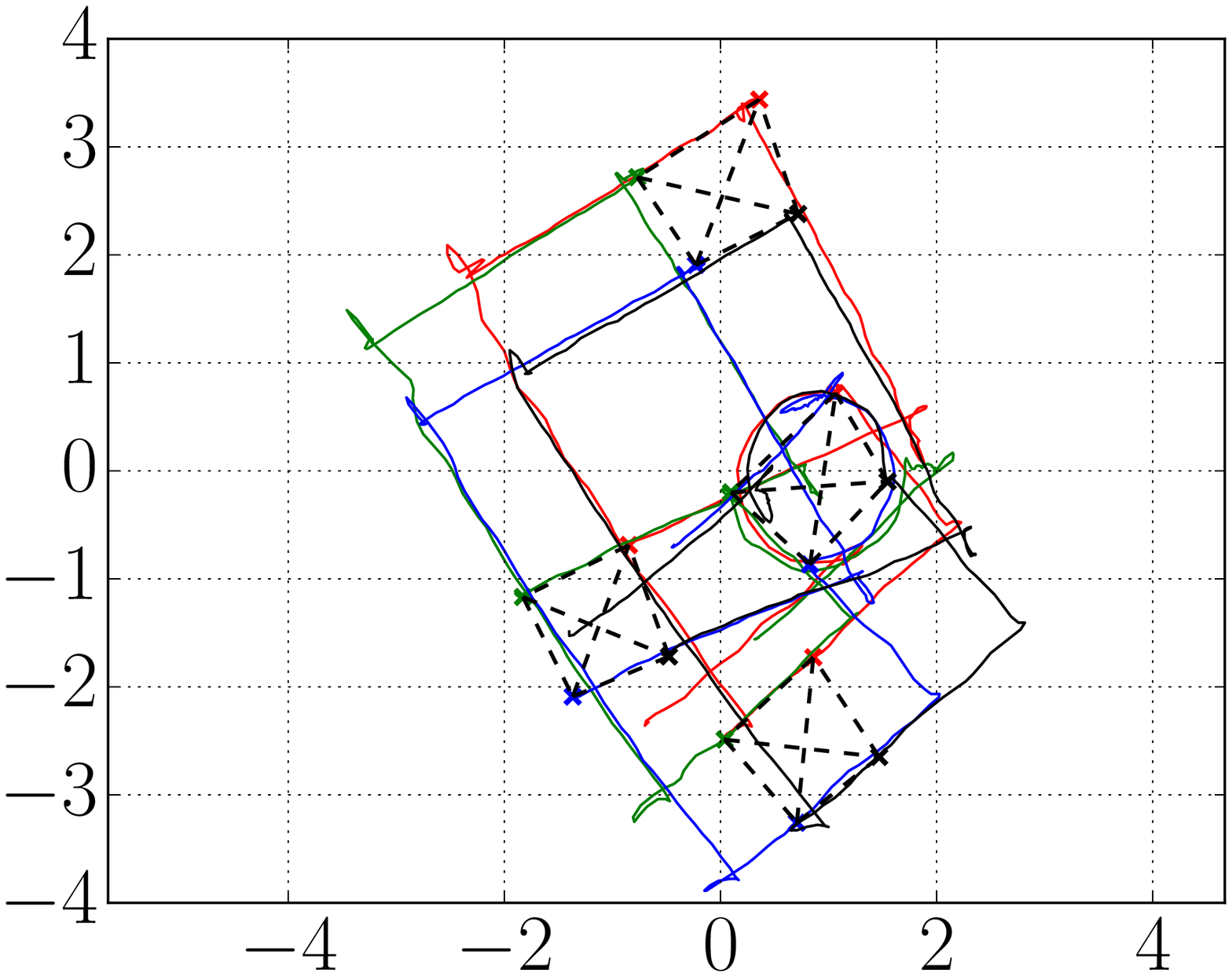}
	\caption{\small On the left, a caption of the rotorcraft formation transporting the payload. On the right, the trajectories (in meters) described by the team of rotorcraft during the collaborative transportation.} 
\label{fig: cap}
\end{figure}



\begin{figure}
\centering
	\includegraphics[width=1\columnwidth]{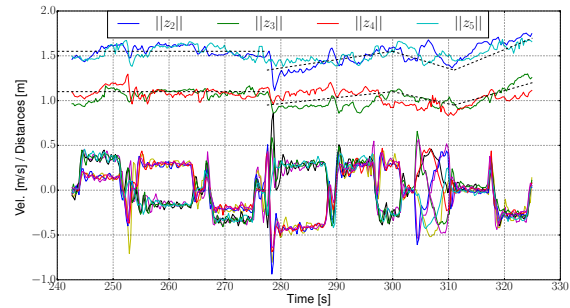}
	\caption{\small Distances (on top) and X/Y velocity components of the vehicles (on bottom) evolution during the collaborative transportation. Black-dashed lines represent the desired inter-vehicle distances. For the translational motion, all vehicles experience the same velocities for the $x$ and $y$ coordinates respectively. In the seconds 253 and 278 the vehicles are asked to change abruptly their velocities. The velocities describe a sinusoid for the spinning motion in the second 305 (circle on the left in Figure \ref{fig: cap}).}
	\label{fig: sig}
\end{figure}

\section{Conclusions}
\label{sec: con}
This paper has experimentally validated the predicted stability properties of a combined method (motion by disagreements + INDI) for the collaborative transportation of a payload by a team of rotorcraft. In particular, the whole setup allows us to perform an accurate analysis of forces and accelerations in the system. Consequently, this approach enable the possibility of performing a worst-case analysis such that a nominal operation can be guaranteed. Experiments are conducted with a team of four rotorcraft carrying a heavy payload, impossible to be lifted by a single vehicle. The experimental results match with the predicted nominal behaviors. 

\bibliographystyle{IEEEtran}
\bibliography{hector_ref.bib}

\begin{thebibliography}{10}
\providecommand{\url}[1]{#1}
\csname url@samestyle\endcsname
\providecommand{\newblock}{\relax}
\providecommand{\bibinfo}[2]{#2}
\providecommand{\BIBentrySTDinterwordspacing}{\spaceskip=0pt\relax}
\providecommand{\BIBentryALTinterwordstretchfactor}{4}
\providecommand{\BIBentryALTinterwordspacing}{\spaceskip=\fontdimen2\font plus
\BIBentryALTinterwordstretchfactor\fontdimen3\font minus
  \fontdimen4\font\relax}
\providecommand{\BIBforeignlanguage}[2]{{%
\expandafter\ifx\csname l@#1\endcsname\relax
\typeout{** WARNING: IEEEtran.bst: No hyphenation pattern has been}%
\typeout{** loaded for the language `#1'. Using the pattern for}%
\typeout{** the default language instead.}%
\else
\language=\csname l@#1\endcsname
\fi
#2}}
\providecommand{\BIBdecl}{\relax}
\BIBdecl

\bibitem{yang2018grand}
G.-Z. Yang, J.~Bellingham, P.~E. Dupont, P.~Fischer, L.~Floridi, R.~Full,
  N.~Jacobstein, V.~Kumar, M.~McNutt, R.~Merrifield \emph{et~al.}, ``The grand
  challenges of science robotics,'' \emph{Science Robotics}, vol.~3, no.~14,
  2018.

\bibitem{wang2004control}
Z.~Wang, Y.~Hirata, and K.~Kosuge, ``Control a rigid caging formation for
  cooperative object transportation by multiple mobile robots,'' in
  \emph{Robotics and Automation, 2004. Proceedings. ICRA'04. 2004 IEEE
  International Conference on}, vol.~2.\hskip 1em plus 0.5em minus 0.4em\relax
  IEEE, 2004, pp. 1580--1585.

\bibitem{AnYuFiHe08}
B.~D.~O. Anderson, C.~Yu, B.~Fidan, and J.~Hendrickx, ``Rigid graph control
  architectures for autonomous formations,'' \emph{IEEE Control Systems
  Magazine}, vol.~28, pp. 48--63, 2008.

\bibitem{MaJaCa15}
H.~G. de~Marina, B.~Jayawardhana, and M.~Cao, ``Distributed rotational and
  translational maneuvering of rigid formations and their applications,''
  \emph{IEEE Transactions on Robotics}, vol.~32, no.~3, pp. 684--697, June
  2016.

\bibitem{smeur2015adaptive}
E.~J. Smeur, Q.~Chu, and G.~C. de~Croon, ``Adaptive incremental nonlinear
  dynamic inversion for attitude control of micro air vehicles,'' \emph{Journal
  of Guidance, Control, and Dynamics}, vol.~38, no.~12, pp. 450--461, 2015.

\bibitem{jiang2013inverse}
Q.~Jiang and V.~Kumar, ``The inverse kinematics of cooperative transport with
  multiple aerial robots,'' \emph{IEEE Transactions on Robotics}, vol.~29,
  no.~1, pp. 136--145, 2013.

\bibitem{sreenath2013dynamics}
K.~Sreenath and V.~Kumar, ``Dynamics, control and planning for cooperative
  manipulation of payloads suspended by cables from multiple quadrotor
  robots,'' \emph{rn}, vol.~1, no.~r2, p.~r3, 2013.

\bibitem{michael2011cooperative}
N.~Michael, J.~Fink, and V.~Kumar, ``Cooperative manipulation and
  transportation with aerial robots,'' \emph{Autonomous Robots}, vol.~30,
  no.~1, pp. 73--86, 2011.

\bibitem{coppola2018board}
M.~Coppola, K.~N. McGuire, K.~Y. Scheper, and G.~C. de~Croon, ``On-board
  communication-based relative localization for collision avoidance in micro
  air vehicle teams,'' \emph{Autonomous Robots}, pp. 1--19, 2018.

\bibitem{tagliabue2017collaborative}
A.~Tagliabue, M.~Kamel, S.~Verling, R.~Siegwart, and J.~Nieto, ``Collaborative
  transportation using mavs via passive force control,'' in \emph{Robotics and
  Automation (ICRA), 2017 IEEE International Conference on}.\hskip 1em plus
  0.5em minus 0.4em\relax IEEE, 2017, pp. 5766--5773.

\bibitem{masone2016cooperative}
C.~Masone, H.~H. B{\"u}lthoff, and P.~Stegagno, ``Cooperative transportation of
  a payload using quadrotors: A reconfigurable cable-driven parallel robot,''
  in \emph{Intelligent Robots and Systems (IROS), 2016 IEEE/RSJ International
  Conference on}.\hskip 1em plus 0.5em minus 0.4em\relax IEEE, 2016, pp.
  1623--1630.

\bibitem{lee2018geometric}
T.~Lee, ``Geometric control of quadrotor uavs transporting a cable-suspended
  rigid body,'' \emph{IEEE Transactions on Control Systems Technology},
  vol.~26, no.~1, pp. 255--264, 2018.

\bibitem{oh2015survey}
K.-K. Oh, M.-C. Park, and H.-S. Ahn, ``A survey of multi-agent formation
  control,'' \emph{Automatica}, vol.~53, pp. 424--440, 2015.

\bibitem{de2017taming}
H.~G. de~Marina, B.~Jayawardhana, and M.~Cao, ``Taming mismatches in
  inter-agent distances for the formation-motion control of second-order
  agents,'' \emph{IEEE Transactions on Automatic Control}, vol.~63, pp.
  449--462, 2018.

\bibitem{garcia2015controlling}
H.~G. de~Marina, M.~Cao, and B.~Jayawardhana, ``Controlling rigid formations of
  mobile agents under inconsistent measurements,'' \emph{Robotics, IEEE
  Transactions on}, vol.~31, no.~1, pp. 31--39, 2015.

\bibitem{smeur2018cascaded}
E.~J. Smeur, G.~de~Croon, and Q.~Chu, ``Cascaded incremental nonlinear dynamic
  inversion for mav disturbance rejection,'' \emph{Control Engineering
  Practice}, vol.~73, pp. 79--90, 2018.

\end{thebibliography}

\end{document}